\documentclass[11pt]{article}
\usepackage{acl2016}
\usepackage{graphicx,amsmath,amssymb,url,xspace,changepage,underscore,booktabs}
\usepackage{amsthm}
\makeatletter
\newtheorem*{rep@theorem}{\rep@title}
\newcommand{\newreptheorem}[2]{%
\newenvironment{rep#1}[1]{%
 \def\rep@title{#2 \ref{##1}}%
 \begin{rep@theorem}}%
 {\end{rep@theorem}}}
\makeatother

\usepackage{times}
\usepackage{latexsym}
\usepackage{fancyvrb}
\usepackage{microtype}
\IfFileExists{dont_upload_to_arxiv}{%
\usepackage[]{todonotes}
}{%
\usepackage[disable]{todonotes} 
}

\newcommand{\Todo}[1]{\todo[author=Pushpendre,size=\small,inline]{#1}}
\newcommand{\union}{\cup}
\IfFileExists{dont_upload_to_arxiv}{
}{
\aclfinalcopy 
\def\aclpaperid{1} 
}

\newcommand{\remove}[1]{} 

\newcommand{\eg}{e.g.,\xspace}

\newcommand{\ie}{i.e.,\xspace}
\newcommand{\bigie}{I.e.,\xspace}
\newtheorem{theorem}{Theorem}
\newreptheorem{theorem}{Theorem}
\newtheorem{lemma}[theorem]{Lemma}
\newreptheorem{lemma}{Lemma}

\newenvironment{definition}[1][Definition]{\begin{trivlist}
\item[\hskip \labelsep {\bfseries #1}]}{\end{trivlist}}

\usepackage{thmtools}
\declaretheoremstyle[%
  spaceabove=-1pt,%
  spacebelow=6pt,%
  headfont=\normalfont\itshape,%
  postheadspace=1em,%
  qed=\qedsymbol%
]{mystyle}
\declaretheorem[name={Proof},style=mystyle,unnumbered,
]{prf}

\newcommand{\figref}[1]{Figure~\ref{#1}}
\newcommand{\tabref}[1]{Table~\ref{#1}}
\newcommand{\thref}[1]{Theorem~\ref{#1}}
\newcommand{\lemref}[1]{Lemma~\ref{#1}}
\newcommand{\secref}[1]{Section~\ref{#1}}
\newcommand\numberthis{\addtocounter{equation}{1}\tag{\theequation}}
\title{A Critical Examination of RESCAL for Completion of Knowledge Bases with Transitive Relations}

\makeatletter
\renewcommand*{\@fnsymbol}[1]{\ensuremath{\ifcase#1\or \dagger\or \ddagger\or
   \mathsection\or \mathparagraph\or \|\or **\or \dagger\dagger
   \or \ddagger\ddagger \else\@ctrerr\fi}}
\makeatother
\IfFileExists{dont_upload_to_arxiv}{
\author{}
}{
\author{Pushpendre Rastogi\thanks{\texttt{pushpendre@jhu.edu}} \and Benjamin Van Durme \\
  Johns Hopkins University}
}
\date{}

\begin{document}

\maketitle
\begin{abstract} Link prediction in large knowledge graphs has received a lot of
attention recently because of its importance for inferring missing relations and
for completing and improving noisily extracted knowledge graphs. Over the years
a number of machine learning researchers have presented various models for
predicting the presence of missing relations in a knowledge base.  Although all
the previous methods are presented with empirical results that show high
performance on select datasets, there is almost no previous work on
understanding the connection between properties of a knowledge base and the
performance of a model. In this paper we analyze the RESCAL
method~\cite{nickel2011three} and show that it can not encode \textit{asymmetric
transitive} relations in knowledge bases.
\end{abstract}

\section{Introduction}
\label{sec:introduction}
Large-scale and highly accurate knowledge bases~(KB) such as
Freebase~\cite{bollacker2008freebase} and YAGO2~\cite{hoffart2013yago2},
have come to be recognized as essential for high
performance on various Natural Language Processing(NLP) tasks. Relation extraction,
Question Answering~\cite{dalton2014entity,fader2014open,yao2014information}
and Entity
Recognition/Disambiguation in informal domains~\cite{ritter2011named,zheng2012entity}
are a few examples of tasks where KBs have
proved to be invaluable. As these
examples demonstrate, increasing the recall of knowledge bases without
compromising on the precision has a direct impact on several tasks that are the
focus of NLP research. Because of this importance of high recall in
knowledge bases and because the recall of even Freebase, the largest open source
KB, is still quite low\footnote{It was reported by~\newcite{dong2014knowledge}
in October 2013, that 71\% of people in Freebase had no known place of birth and
that 75\% had no known nationality.} a number of researchers have published
heuristics with their empirical performance on automatically inferring the
information that is missing in knowledge bases. Unfortunately the literature
on theoretical analysis for these methods is still scarce.

In this paper we analyze RESCAL~\cite{nickel2011three} which
is a widely cited method for inferring missing relations in KBs.
The RESCAL method embeds entities and relations in a KB using vectors and matrices
respectively and it predicts the true status of en edge between two nodes
using these representations. Although RESCAL was introduced in 2011
and has been shown to be effective on a variety of
datasets~\cite{toutanova2015representing,nickel2011three,nickel2012factorizing}
there has been no theoretical analysis of the failure modes of this method.
We show , both theoretically and experimentally (Sections~\ref{sec:main-result}
and \ref{sec:experiment}),
that RESCAL is not suitable for predicting missing relations
in a KB that contains transitive and asymmetric relations such as
the ``type of'' relation which is very important in
Freebase~\cite{guha2015towards} and the ``hypernym'' relation which is important in
WordNet~\cite{miller1995wordnet}.

\section{Analysis of RESCAL}
\label{sec:main-result}
\paragraph{Notation:} A
knowledge base contains, but is not equal to, a collection of
\textit{(subject, relation, object)}
triples. Each triple encodes the fact that a \textit{subject} entity is related to
an \textit{object} through a particular type of \textit{relation}.  Let
$\mathcal{V}$ and $\mathcal{R}$ denote the finite set of entities and
relationships. We assume that $\mathcal{R}$ includes a type for the \textit{null
relation} or \textit{no relation}. Let $\mathrm{V} = |\mathcal{V}|$ and
$\mathrm{R}=|\mathcal{R}|$ denote the number of entities and relations.
We use $v$ and $r$ to denote a generic entity and relation respectively.
The shorthand $[n]$ denotes $\{x | 1 \le x \le n, x \in \mathbb{N}\}$.
Let $\mathrm{E}$ be the number of triples known to us and let $e$ denote a
generic triple. We denote the
subject, object and relation of $e$ through $e^{sub} \in \mathcal{V}$, $e^{obj}
\in \mathcal{V}$ and $e^{rel} \in \mathcal{R}$ respectively and we denote the
entire collection of facts as $\mathcal{E} = \{e_k | k \in [\mathrm{E}]\}$.

\paragraph{RESCAL:} The RESCAL model associates each entity $v$ with the
vector $a_v \in \mathbb{R}^d$ and it represents the relation $r$ through the
matrix $M_r \in \mathbb{R}^{d \times d}$. Let $v$ and $v'$ denote two entities whose
relationship is unknown, the RESCAL model predicts the relation between $v$ and $v'$
to be:
\begin{align}
  \hat{r} &= \operatorname*{argmax}_{r \in \mathcal{R}} s(v, r, v')\\
  s(v, r, v') &= s(v, r, v') = a_v^{{T}} M_r a_{v'}  \label{eq:score}
\end{align}
Note that in general if the matrix $M_r$ is asymmetric then the score function
$s$ would also be asymmetric, \ie $s(v, r, v') \ne s(v', r, v)$.
Let $\Theta = \{a_v | v \in \mathcal{V}\} \union{} \{M_r | r \in \mathcal{R}\}$.
Clearly $\Theta$ parameterizes RESCAL.
Therefore even though the same embedding is used to represent an entity regardless of whether it is the first or the second entity in a relation, the RESCAL model could still
handle asymmetric relations if the matrix $M_r$ is asymmetric.
\paragraph{Transitive Relations and RESCAL:} In addition to relational
information about the binary connections between entities, many KBs
contain information about the relations themselves.
For example, consider the toy knowledge base depicted in \figref{fig:toykb}.
Based on the information that \texttt{Fluffy} \textit{is-a} \texttt{Dog} and that a
\texttt{Dog} \textit{is-a} \texttt{Animal} and that \textit{{is-a}}
is a transitive relations we can infer missing relations such as
\texttt{Fluffy} \textit{is-a} \texttt{Animal}.
\begin{figure}[htbp]
  \centering
  \IfFileExists{fluffy-crop.pdf}{
    \includegraphics[width=.6\columnwidth]{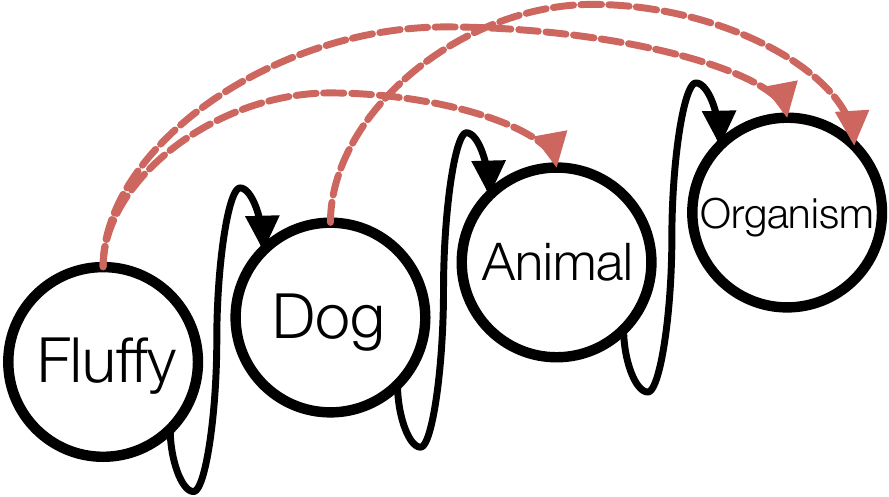}}{
    \includegraphics[width=.6\columnwidth]{figure/fluffy-crop.pdf}}
  \caption{\footnotesize  A toy knowledge base containing only \textit{is-a} relations. The
dashed edges indicate unobserved relations that can be recovered using the
observed edges and the fact that \textit{is-a} is a transitive relation.}
  \label{fig:toykb}
\end{figure}

Let us now analyze what happens when we encode a transitive, asymmetric
relation with RESCAL. Consider the situation where the set
$\mathcal{R}$ only contains two relations $\{r_0, r_1\}$.
$r_1$ denotes the presence of the \textit{is-a}
relation and $r_0$ denotes the absence of that relation.
The RESCAL model can only follow the chain of transitive relations
and infer missing edges using existing information in the graph
if for all triples of vertices $v, v', v''$
in $\mathcal{V}$ for which we have observed \textit{($v$, is-a, $v'$)} and
\textit{($v'$, is-a, $v''$)} the following holds true:
\begin{adjustwidth}{-.5em}{}
\begin{align*}
& s(v, r_1, v') > s(v, r_0, v') \land s(v', r_1, v'') > s(v', r_0, v'') \\
& \implies s(v, r_1, v'') > s(v, r_0, v'')
\end{align*}
\end{adjustwidth}
This can be rewritten as:
\begin{align*}
& \forall v, v', v'' \in \mathcal{V} \\
& a_v^{{T}} (M_{r_1} - M_{r_0}) a_{v'} > 0 \land a_{v'}^{{T}} (M_{r_1} - M_{r_0}) a_{v''} > 0 \\
& \implies a_{v}^{{T}} (M_{r_1} - M_{r_0}) a_{v''} > 0 \numberthis \label{eq:mtrans}
\end{align*}

\noindent We now define a \textit{transitive matrix} and
and state a theorem that we prove in the Appendix.
\begin{definition} We say that a matrix $M \in \mathbb{R}^{d \times d}$ is
transitive if every triple of vectors $a,b,c \in \mathbb{R}^d$ that satisfy
$a^{{T}} M b > 0$ and $b^{{T}} M c > 0$ also satisfy $a^{{T}} M c > 0$.
\end{definition}

\begin{theorem}\label{thmMain}
  Every transitive matrix is symmetric.
\end{theorem}

If we enforce the constraint in Equation~\ref{eq:mtrans} to hold
for all possible vectors and not just a finite number of vectors
then $M_{r_1} - M_{r_0}$ is a transitive matrix.
By \thref{thmMain} $M_{r_1} - M_{r_0}$ must be symmetric.
This implies that if the RESCAL
model predicts that $s(v, r_1, v') > s(v, r_0, v')$ then it would also predict
that $s(v', r_1, v) > s(v', r_0, v)$. In terms of the toy KB shown in
\figref{fig:toykb}; if the RESCAL model predicts that \texttt{Fluffy} \textit{is-a}
\texttt{Animal}
then it would also predict that \texttt{Animal} \textit{is-a} \texttt{Fluffy}.
Therefore the RESCAL model
is not suitable for encoding assymmetric, transitive relations.

\section{Experiments}
\label{sec:experiment}
During our analysis in \secref{sec:main-result},
we made assumption that the constraint of
equation~\ref{eq:mtrans} held over all vectors in
$\mathbb{R}^d$ instead of just a finite number of vector triples.
This assumption was used to make conclusions about
RESCAL using \thref{thmMain}.

A fair criticism of our analysis is that practically the
RESCAL model only needs to encode a finite number of vertices into vector space
and it is possible that there exists an asymmetric matrix that can correctly
make the finite number of deductions that are possible inside a finite KB. This could
be especially true when the dimensionality $d$ of the RESCAL embeddings is high.
On the other hand, it is intuitive that as the number of entities
inside a KB increases our assumptions and analysis would become
increasingly better approximations of reality. Therefore the performance of the
RESCAL model should degrade as the number of entities inside the KB increases
and the dimensionality of the embeddings remains constant.

\subsection{On Simulated Data}
\label{sec:simulation-study}
In order to test the applicability of our analysis we performed the following
experiment:  We started with a complete, balanced, rooted, directed binary tree
$\mathcal{T}$, with edges directed \textit{from} the root \textit{to} its
children. We then augmented $\mathcal{T}$ as follows: For every tuple of
distinct vertices $v$, $v'$ we added a new edge to $\mathcal{T}$ if there
already existed a directed path starting at $v$ and ending at $v'$ in
$\mathcal{T}$. We stopped when we could not add any more edges without creating
multi-edges.  For the rest of the paper we denote this resulting set of ordered
pairs of vertices as $\mathcal{E}$ and those pairs of vertices that are not in
$\mathcal{E}$ as $\mathcal{E}^c$.  For example $\mathcal{E}$ contains an edge
from the root vertex to every other vertex and $\mathcal{E}^c$ contains an edge
from every vertex to the root vertex. For a tree of depth 11, $\mathrm{V} =
2047, \mathrm{E}=18,434$ and $|\mathcal{E}^c| = 4,171,775$.

We trained the RESCAL model under two settings: In the first setting we used
entire $\mathcal{E}$ and $\mathcal{E}^c$ as training inputs to the RESCAL model.
We denote this setup as \textit{FullSet}. In the second setting
we randomly sample $\mathcal{E}^c$ and select only $\mathrm{E} = |\mathcal{E}|$
edges from $\mathcal{E}^c$.  We denote this training setup as
\textit{SubSet}.
Note however, that all in the edges in $\mathcal{E}$ including all the edges
in the original tree are always used during both \textit{FullSet} and
\textit{SubSet}.

For both the settings of \textit{FullSet} and \textit{SubSet}
we trained the RESCAL model 5 times and evaluated the models'
predictions on the following three subsets of the edges:
$\mathcal{E}$, $\mathcal{E}^c$ and $\mathcal{E}^{(rev)}$.
$\mathcal{E}$ and $\mathcal{E}^c$ were introduced earlier.
To recall, $\mathcal{E}$ contains all
ordered pairs of vertices that are in the transitive relation of being
connected, $\mathcal{E}^c$ contains pairs
of vertices that are not connected and not in a relation.  $\mathcal{E}^{(rev)}$
denotes the set of ordered pairs whose reverse pair exists in $\mathcal{E}$.  \bigie
$\mathcal{E}^{rev} = \{(u,v) | (v,u) \in \mathcal{E}\}$.
For every edge in each of these subsets we evaluate the model's performance under
$0-1$ loss. For example, when we
evaluate the performance of RESCAL on an edge $(v, v') \in \mathcal{E}$ we
evaluate whether the model assigns a higher score to $(v, r_1, v')$ than
$(v, r_0, v')$ and reward the model by $1$ point if it makes the right
prediction and $0$ otherwise.
As before, $r_1$ and $r_0$ denote the presence and abscence of relationship between $v$
and $v'$.

We note that low Performance on $\mathcal{E}^{rev}$ and high performance on
$\mathcal{E}$ would indicate exactly the kind of failure that we predicted from
our analysis.

As explained earlier, the dimensionality of the RESCAL embedding, $d$, and the
number of entities, $\mathrm{V}$ significantly influence the performance of
RESCAL therefore we vary them and tabulate the results in
\tabref{tab:nns} and \ref{tab:ns}.
\IfFileExists{dont_upload_to_arxiv}{
  \footnote{We have submitted our experiment script as well.}
}{}

\begin{table}[htbp]
  \centering
  \begin{tabular}{l | c | c | c}
$d$    & $\mathrm{V}=2047$ & $4095$     & $8191$     \\\hline
$50$   & 66 100 100        & 60 100 100 & 54 100 100 \\
$100$  & 76 100 100        & 69 100 100 & 63 100 100 \\
$200$  & 86 100 100        & 79 100 100 & 72 100 100 \\
$400$ & 94 100 100        & 88 100 100 & 81 100 100 \\
  \end{tabular}
  \caption{\footnotesize Percentage accuracy of RESCAL with \textit{FullSet}.
    Every table element is a triple of numbers measuring the performance of
    RESCAL on $\mathcal{E}, \mathcal{E}^c, \mathcal{E}^{rev}$ respectively.
    $\mathrm{V}$ denotes the number of nodes in the tree and $d$ denotes the
    number of dimensions used to parameterize the entities.}
  \label{tab:nns}
\end{table}

\begin{table}[htbp]
  \centering
  \begin{tabular}{l| c| c | c}
    $d$    & $\mathrm{V}=2047$ & $4095$    & $8191$    \\\hline
    $50$   & 100 93 52         & 100 91 48 & 100 89 44 \\
    $100$  & 100 78 58         & 100 92 56 & 100 89 52 \\
    $200$  & 100 60 72         & 100 71 61 & 100 90 59 \\
    $400$ & 100 54 67         & 100 57 70 & 100 65 62 \\
  \end{tabular}
  \caption{\footnotesize Accuracy of RESCAL trained with \textit{SubSet}.}
  \label{tab:ns}
\end{table}

\subsection{On WordNet}
\label{sec:wordnet-experiments}
In order to test our analysis on real data we performed experiments on the WordNet dataset. WordNet contains vertices called synsets that are arranged in a tree like hierarchy under the relation of \textit{hyponymy}. For example, a \textit{dog} is a hyponym of \textit{animal} and an animal is a hyponym of \textit{living\_thing} therefore a \textit{dog} is a hyponym of \textit{living\_thing}. To conduct our experiments we extracted all the hyponyms of the \textit{living\_thing} synset as a tree and edges to that tree to form a transitive closure under the hyponym relation. The \textit{living\_thing} synset contained $16255$ hyponyms which were connected with $16489$ edges and after performing the transitive closure the number of edges became $128241$, \ie $\mathrm{V} = 16,255$ and $\mathrm{E} =128,241$.
We performed two experiments under the \textit{FullSet} and \textit{SubSet} protocols in exactly the same way as described in \secref{sec:simulation-study} with the new graph.
The results, shown in \tabref{tab:wordnet}, exhibit the same trends as seen in Table~\ref{tab:nns} and \ref{tab:ns}. See the following section for a more thorough discussion of results.
  \begin{table}[htbp]
    \centering
    \begin{tabular}{l | c || c}
      $d$     & \textit{FullSet} & \textit{SubSet} \\\hline
      $50$    & 71 100 100&100 93 58\\
      $100$   & 79 100 100&100 94 60\\
      $200$   & 84 100 100&100 93 63\\
      $400$   & 89 100 100 & 100 68 69\\
    \end{tabular}
    \caption{Results from experiments on WordNet.
      Specifically we chose to use the subtree
      rooted at the \textit{living_things} synset from the WordNet hierarchy.
      Every synset in the subtree corresponds to a vertex. Consequently,
      for all our experiments $\mathrm{V}=16413$.}
    \label{tab:wordnet}
  \end{table}

\section{Related Work}
\label{sec:related-work}
Most previous works for inferring the missing information in knowledge bases
assumes that a knowledge base is just a graph with labeled
vertices and labeled edges~\cite{nickel2016review,toutanova2015representing} and
they either focus on inferring which labeled edge, if any, should be used to
connect two previously unconnected vertices or they try to learn what vertex
label/entity type, if any, should be used to annotate an unlabeled entity.

The task of
predicting missing edges in a KB, which we focus on, has previously been
called Link Prediction~\cite{liben2007link,nickel2011three}, Knowledge Base
Completion (KBC)~\cite{socher2013reasoning,west2014knowledge} or more broadly
Relational Machine Learning~\cite{nickel2016review}.\footnote{The terminology
used for the task of inferring missing vertex labels, which is not the focus of
this paper, is even more diverse. This task has been termed Class/Labeled
Instance acquisition~\cite{van2008finding,talukdar2010experiments}, Collective
Classification~\cite{sen2008collective} and Vertex
Nomination~\cite{fishkind2015vertex}.}  Besides the
before-mentioned papers the following publications also present models for
KBC which we list without comments
\cite{bordes2011learning,Lao2011random,gardner2015efficient,lin2015learning,Zhao2015knowledge,Wang2015knowledge,he2015knowledge,Wei2015large}.
\Todo{
a study of RESCAL and other
tensor factorization techniques was already studied in "Towards Combined Matrix
and Tensor Factorization for Universal Schema Relation Extraction"

So what? Their method of analysis and conclusions are very different.
Do I need to qualify that this paper is different?}

\Todo{Issues with
RESCAL-like algorithms have been known for a while and many papers proposed
improvements and refinements. These should be discussed and considered like the
introduction of typing for instance:

Chang, K. W., Yih, W. T., Yang, B., and Meek, C. (2014, July). Typed Tensor
Decomposition of Knowledge Bases for Relation Extraction. In EMNLP (pp.
1568-1579).

But this paper intends to use type information during training. My analysis
is orthogonal to it.}

\section{Results and Discussion}
\label{sec:results}
\tabref{tab:nns} and~\ref{tab:ns} show the performance of
the RESCAL model, for encoding three
subsets of relational information, $\mathcal{E}, \mathcal{E}^c$ and
$\mathcal{E}^{rev}$ in increasingly large KBs with a single transitive
relation under a broad range of settings.

The results in \tabref{tab:nns} were obtained by feeding RESCAL
$\mathcal{E} \cup \mathcal{E}^c$ as training data. Note that
RESCAL received all possible information during training so we are evaluating the
training accuracy of the model at this point. Low accuracy under this
setting implies that the model does not have the capacity to learn the rules in
the knowledge base. We observe that the accuracy
of RESCAL decreases as the number of entities, $\mathrm{V}$ increases and it
increases as the dimensionality, $d$ increases which in line with our
predictions. We also note that since
$\mathcal{E}^c$ is much larger than $\mathcal{E}$ therefore the
training objective of RESCAL favors good performance on $\mathcal{E}^c$ and
accordingly the accuracy of RESCAL on edges in $\mathcal{E}^c$ remains high but the
performance on $\mathcal{E}$ suffers. 
The high accuracy of RESCAL with $\mathrm{V}=2047$ and $d=400$ suggests
that with a high enough dimensionality of the
embeddings it is possible to embed a finite database with high
accuracy. But increasing the dimensionality
of RESCAL embeddings can become infeasible for an extremely large
knowledge base. Also we can observe that the performance of RESCAL degrades as
the number of entities inside the KB increases and the dimensionality of the
embeddings remains constant.

The results in \tabref{tab:ns} were obtained by training RESCAL with
$\mathcal{E}$ and a subset of $\mathcal{E}^c$. This training method is closer to
the way such embedding based methods for KBC are usually trained~\cite{nickel2016review}. We observe that the accuracy of the
RESCAL model on $\mathcal{E}^{rev}$ is substantially lower than its performance
on either $\mathcal{E}$ or $\mathcal{E}^c$, especially in the upper triangle
region of the table where $\mathcal{V}$ is high and $d$ is low.
This result is in accordance with
our analysis that under the RESCAL mode if $s(v, r_1, v') > s(v, r_0, v')$ then
$s(v', r_1, v) > s(v', r_0, v)$ as well. Our results also highlight a problem with
the commonly employed KBC evaluation protocol
of randomly dividing the edge set of a graph into train and test sets for measuring
knowledge base completion accuracy.
For example with $d=50$ the average accuracy on both $\mathcal{E}$
and $\mathcal{E}^c$ is quite high but on
$\mathcal{E}^{rev}$ accuracy is low even though $\mathcal{E}^{rev}$ is a subset of
$\mathcal{E}^c$. Such a failure would stay undetected with existing
evaluation methods.

\Todo{It would be more interesting if the authors can discuss alternative tensor
decomposition methods that might be suitable for modeling the transitive
asymmetric relations.
       - It seems to me the same argument also holds for other tensor
factorization methods with similar structure such as DEDICOM}
\section{Conclusions}
\label{sec:conclusions}
In this paper we investigated a popular KBC algorithm
named RESCAL and through our analysis of the scoring function employed in
RESCAL and our experiments on simulated data, we showed that the
RESCAL method does not
perform well in encoding transitive and asymmetric relations and specifically
that its inferences about edges that are the reverse of edges that are present in a
knowledge have a high chance of being incorrect. Although our analysis relied on
somewhat strong assumptions that the constraint in equation~\ref{eq:score}
holds true over all points in the vector space we showed that the insights
gained were useful in practice.

One of the key idea underlying
our work was that knowledge bases should be considered as more than just graphs
since KBs also contain logical structure amongst the predicates.
By taking such logical structure, \eg the constraint that if vertex $v$ connects to $v'$ and $v'$ connects to $v''$ then $v$ connects to $v''$, to a logical extreme we came up with a well founded argument about the performance of RESCAL in encoding knowledge bases with transitive relations.
We believe that this idea can be gainfully used to analyze other KBC methods as well.


\appendix
\section{Proof of Theorem~\ref{thmMain}}
\label{sec:proof}
We note that \thref{thmMain} was first proven by \newcite{212808}.
Here we give an alternative proof.
\begin{lemma}\label{lem2}
  Every transitive matrix is PSD.
\end{lemma}
\begin{prf}
  Consider the triplet of vectors $c := x, b := Mc, a := Mb$. Then
  $a^T(Mb) = ||Mb||^2 \ge 0$ and $b^T(Mc) = ||b||^2 \ge 0$ and $a^TMc = b^TMb$.
  Either $b = 0$ or $b \ne 0$ and $Mb = 0$, or both $Mb \ne 0$ and $b
  \ne 0$ which implies $b^TMb > 0$ (by transitivity). In all three
  cases $b^TMb \ge 0$.
\end{prf}

\begin{lemma}\label{lem1}
  Let $M_1, M_2 \in \mathbb{R}^{d \times d} \setminus \left\{ 0 \right\}$. If $\
  \forall x, y : x^TM_1y > 0 \implies x^TM_2y > 0$
  then $M_1 = \lambda M_2$ for some $\lambda > 0$.
\end{lemma}
We defer the proof of this technical lemma to the supplementary material
submitted with the paper.
\begin{lemma}\label{lem3}
  If $\exists x, y\ x^TMy > 0 $ but $x^TM^Ty < 0$ then $M$ is not transitive.
\end{lemma}
\begin{prf}
  Let $x,y$ be two vectors that satisfy $x^TMy > 0$ and $x^TM^Ty < 0$. Since
  $x^TM^Ty = y^TMx$ therefore $y^T M (-x) > 0$. If we assume $M$ is transitive,
  then $x^T M (-x) > 0$ by transitivity , but \lemref{lem2} shows such an $x$ cannot
  exist.
\end{prf}

\begin{reptheorem}{thmMain}
Every transitive matrix is symmetric.
\end{reptheorem}
\begin{prf}
  By Lemma~\ref{lem3} $x^TMy > 0 \implies x^TM^Ty > 0$. Using \lemref{lem1} we
  get $M = \lambda M^T$ for some $\lambda > 0$. Clearly $\lambda = 1$.
\end{prf}

\bibliography{antirescal}
\bibliographystyle{acl2016}

\clearpage
\section*{Supplementary Material}
Before proving \lemref{lem1} let us present its analogue for vectors.
\begin{lemma}\label{lem0}
  Let $x, y \in \mathbb{R}^d \setminus \left\{ 0 \right\}$.
  If $\nexists z \in \mathbb{R}^d$ such that
  $x^Tz > 0$ and $y^Tz < 0$ then $x = \lambda y$ for some $\lambda > 0$.
\end{lemma}
\begin{proof}
  If $x = \lambda y$ then $x^{{T}} y = \lambda y^{{T}} y$. Since $y^{{T}} y > 0$
  therefore $\lambda > 0$. In the
  case that $x \ne \lambda y$ then by Cauchy Schwartz inequality $D := (x^{{T}}
y)^2 - (x^{{T}} x)(y^{{T}} y) \ne 0$. Consider the vector $\alpha x + \beta y$
with $\alpha = -\frac{x^{{T}} y + y^{{T}} y}{D}$ and $\beta = \frac{x^{{T}} y +
x^{{T}} x}{D}$. It is easy to check that $(\alpha x + \beta y)^{{T}} x$ and
$(\alpha x + \beta y)^{{T}} y$ equal 1 and $-1$, which contradicts
the hypothesis.
\end{proof}

\begin{replemma}{lem1}
  Let $M_1, M_2 \in \mathbb{R}^{d \times d} \setminus \left\{ 0 \right\}$. If $\
  \forall x, y : x^TM_1y > 0 \implies x^TM_2y > 0$
  then $M_1 = \lambda M_2$ for some $\lambda > 0$.
\end{replemma}
\begin{proof}
  Choose an $x \in \mathbb{R}^d$ for which $x^TM_1 \ne 0$.
  If such an $x$ does not $M_1 = 0$ in contradiction to the hypothesis.
  Note that if $x^TM_1y \ne 0$ then either $x^TM_1y$ or $x^TM_1-y$ would be positive.
  Since $(x^TM_1) y > 0 \implies (x^TM_2) y > 0$ therefore $\nexists y$ for
  which $(x^TM_1) y > 0$ but $(x^TM_2) y < 0$.
  By \lemref{lem0} $x^TM_1 = \lambda_x x^T M_2$. Furthermore from the proof of \lemref{lem0}
  $\lambda_x = \frac{x^TM_1M_2x}{x^TM_2M_2x}$ therefore $\lambda_x$ is continous
  with respect to $x$.
  Now we prove that $\lambda_x$ is constant.
  Consider vectors $x$ and $\alpha x$. As shown earlier, $(\alpha x)^T M_1 =
  \lambda_{\alpha x} (\alpha x)^T M_2$. But $(\alpha x)^T M_1 = \alpha (x^T M_1)
  = \alpha \lambda_x x^T M_2$. Therefore $\lambda_{\alpha x} = \lambda_x$.
  Since $\lambda_x$ is continous at $0$ therefore
  $\lambda_{\alpha x}$ equals the constant $\lambda_{0}$.
  This implies $x^T(M_1 - \lambda_{0} M_2) = 0$.
  Clearly $\lambda = \lambda_{0} > 0$.
\end{proof}
\end{document}